\documentclass[twoside,11pt]{article}
\usepackage{pgm}

\usepackage[utf8]{inputenc}
\inputencoding{utf8}

\usepackage{hyperref,color,soul,booktabs}
\setulcolor{blue}
\newcommand{\BibTeX}{\textsc{B\kern-0.1emi\kern-0.017emb}\kern-0.15em\TeX}

\ShortHeadings{A Proof of the Front-Door Adjustment Formula}{Javidian and Valtorta}

\begin{document}
\title{A Proof of the Front-Door Adjustment Formula}

\author{\Name{Mohammad Ali Javidian} \Email{javidian@email.sc.edu}\and
	\Name{Marco Valtorta} \Email{mgv@cse.sc.edu}\\
	\addr Department of Computer Science \& Engineering, University of South Carolina, Columbia, SC, 29201, USA.}

\maketitle

\begin{abstract}
	We provide a proof of the the Front-Door adjustment formula using the $do$-calculus.  
\end{abstract}
\section{Introduction}

In \citep{pj},  a formula for computing the causal effect of $X$ on $Y$ in the causal model of figure \ref{Fig:fd1} is derived and used  to motivate the definition of front-door criterion.  Pearl then states, without proof, the Front-Door Adjustment Theorem~\citep[Theorem 3.3.4]{pj}. In section 3.4.3, he provides a symbolic derivation of the front door adjustment formula for the same example from the $do$-calculus. In this short technical report, we provide a proof of Theorem 3.3.4 using the $do$-calculus.  
\begin{figure}[h]
	\centering
	\includegraphics[scale=.5]{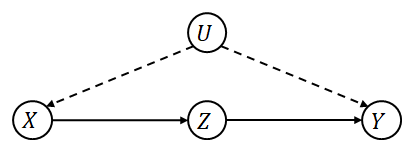}
	\caption{A causal Bayesian network with a latent variable $U$.} \label{Fig:fd1}
\end{figure}
The next section consists of the proof of the front-door adjustment formula; the theorem is restated for the reader's convenience.
The $do$-calculus rules, the back-door criterion, the back-door adjustment formula, and the front-door criterion are in the slide set provided as an ancillary document.

\section{Front-Door Adjustment Theorem}

\begin{theorem}[Front-Door Adjustment]
	If a set of variables $Z$ satisfies the front-door criterion relative to $(X, Y)$ and if $P(x,z)>0$, then the causal
	effect of $X$ on $Y$ is identifiable and is given by the formula
	\begin{equation}\label{fd}
	P(y|\hat{x})=\sum_{z}^{}P(z|x)\sum_{x'}^{}P(y|x',z)P(x')
	\end{equation}.
\end{theorem}
\begin{proof}
	By well known probability identities (for example, the Fundamental Rule and the Theorem of Total Probability), $P(y|\hat{x})= \sum_{z}^{}P(y|z,\hat{x})P(z|\hat{x})$. In Step 1, below, we show how to compute $P(z|\hat{x})$ using only observed quantities.  In Steps 2 and 3, we show how to compute $P(y|z,\hat{x})$ using only observed quantities; this part of the proof is by far the hardest.
	
\begin{itemize}
\item Step 1: Compute $P(z|\hat{x})$
\begin{itemize}
	\item $X\!\perp\!\!\!\perp Z$ in $G_{\underline{X}}$ because there is no outgoing edge from $X$ in $G_{\underline{X}}$, and also by condition (ii) of the definition of the front-door criterion, all back-door paths from $X$ to $Z$ are blocked.
	\begin{figure}[h]
		\centering
		\includegraphics[scale=1.00]{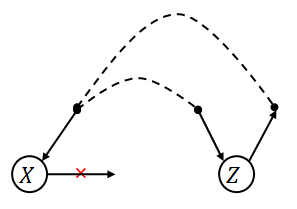}
		\label{Fig:4}
	\end{figure}
	\item $G$ satisfies the applicability condition
	for Rule 2: $$P(y|\hat{x},\hat{z},w)=P(y|\hat{x},z,w)\textrm{\quad if\quad } (Y\!\perp\!\!\!\perp Z|X,W)_{G_{\overline{X}\underline{Z}}}.$$
	\item In Rule 2, set $y=z, x=\o, z=x, w=\o$:
	\begin{equation} \label{step1}
	P(z|\hat{x}) = P(z|x)
	\end{equation} $\textrm	{\quad because \quad} (Z\!\perp\!\!\!\perp X)_{G_{\underline{X}}}$
	
\end{itemize}

\item Step 2:  $P(y|\hat{z})$
\begin{itemize}
\item $P(y|\hat{z})=\sum_{x}^{}P(y|x,\hat{z})P(x|\hat{z}).$

\item $X\!\perp\!\!\!\perp Z$ in $G_{\overline{Z}}$ because there is no incoming edge to $Z$ in $G_{\overline{Z}}$, and also all paths from $X$ to $Z$ either by condition (ii) of the definition of the front-door criterion (blue-type paths), or because of existence of a collider node on the path (green-type paths) are blocked.
\begin{figure}[h]
	\centering
	\includegraphics[scale=1.0]{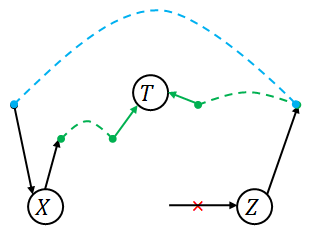}
	\label{Fig:4}
\end{figure}
\item $G$ satisfies the applicability condition
	for Rule 3: $$P(y|\hat{x},\hat{z},w)=P(y|\hat{x},w)\textrm{\quad if\quad } (Y\!\perp\!\!\!\perp Z|X,W)_{G_{\overline{X},\overline{Z(W)}}}.$$
\item In Rule 3, set $y=x, x=\o, z=z, w=\o$:
$$P(x|\hat{z}) = P(x) \textrm{\quad because \quad} (Z\!\perp\!\!\!\perp X)_{G_{\overline{Z}}}.$$
\item $(Z\!\perp\!\!\!\perp Y|X)_{G_{\underline{Z}}}$ because there is no outgoing edge from $Z$ in $G_{\underline{Z}}$, and also by condition (iii) of the definition of the front-door criterion, all back-door paths from $Z$ to $Y$ are blocked by $X$.
\begin{figure}[h]
\centering
\includegraphics[scale=1.00]{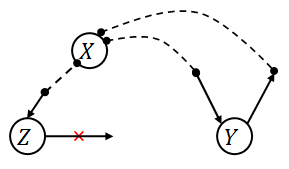}
\label{Fig:5}
\end{figure}
\item $G$ satisfies the applicability condition
for Rule 2: $P(y|\hat{x},\hat{z},w)=P(y|\hat{x},z,w)\textrm{\quad if\quad } (Y\!\perp\!\!\!\perp Z|X,W)_{G_{\overline{X}\underline{Z}}}.$
\item In Rule 2, set $y=y, x=\o, z=z, w=x$:
$$P(y|x,\hat{z}) = P(y|x,z) \textrm{\quad because \quad} (Z\!\perp\!\!\!\perp Y|X)_{G_{\underline{Z}}}.$$

\item
\begin{equation} \label{step2}
P(y|\hat{z})= \sum_{x}^{}P(y|x,\hat{z})P(x|\hat{z})=\sum_{x}^{}P(y|x,{z})P(x)
\end{equation}
This formula is a special case of the back-door formula. 
\end{itemize}

\item Step 3: Compute $P(y|\hat{x})$

As already noted at the beginning of the proof, $P(y|\hat{x})= \sum_{z}^{}P(y|z,\hat{x})P(z|\hat{x}).$
\begin{itemize}
	\item $P(z|\hat{x}) = P(z|x)$, as shown in Step 1 (see equation (\ref{step1}))
\end{itemize}
There is no rule of the $do$-calculus that allows the elimination of the hat from $P(y|z,\hat{x})$, so we take a circuitous route: we first replace an observation ($z$) with an intervention ($\hat{z}$) using Rule 2, and then remove an intervention variable ($\hat{z}$) using Rule 3.
\begin{itemize}
	\item $(Y\!\perp\!\!\!\perp Z|X)_{G_{\overline{X}\underline{Z}}}$ because there is no outgoing edge from $Z$ in $G_{\overline{X}\underline{Z}}$, and also by condition (iii) of the definition of the front-door criterion, all back-door paths from $Z$ to $Y$ are blocked by $X$.
	\begin{figure}[h]
		\centering
		\includegraphics[scale=1.00]{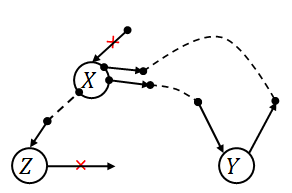}
		\label{Fig:6}
	\end{figure}
	\item $G$ satisfies the applicability condition
		for Rule 2: $P(y|\hat{x},\hat{z},w)=P(y|\hat{x},z,w)\textrm{\quad if\quad } (Y\!\perp\!\!\!\perp Z|X,W)_{G_{\overline{X}\underline{Z}}}.$
	\item In Rule 2, set $y=y, x=x, z=z, w=\o$ :
	$$P(y|z,\hat{x}) = P(y|\hat{z},\hat{x}) \textrm{\quad because \quad} (Y\!\perp\!\!\!\perp Z|X)_{G_{\overline{X}\underline{Z}}}.$$
	\item $(Y\!\perp\!\!\!\perp X|Z)_{G_{\overline{X}\overline{Z}}}$ because there is no incoming edge to $X$ in $G_{\overline{X}\overline{Z}}$, and also all paths from $X$ to $Y$  are blocked either because of condition (i) of the definition of the front-door criterion (blue-type paths)[directed paths from $X$ to $Y$], or because of the existence of a collider on the path (green-type paths) (note that the case $T\in Z$ cannot happen because there is no incoming edge to $Z$ in $G_{\overline{X}\overline{Z}}$).
	\begin{figure}[h]
		\centering
		\includegraphics[scale=1.00]{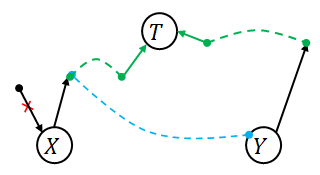}
		\label{Fig:7}
	\end{figure}

	\item $G$ satisfies the applicability condition for Rule 3: $$P(y|\hat{x},\hat{z},w)=P(y|\hat{x},w)\textrm{\quad if\quad } (Y\!\perp\!\!\!\perp Z|X,W)_{G_{\overline{X},\overline{Z(W)}}}.$$
	\item In Rule 3, set $y=y, x=z, z=x, w=\o$:
	$$P(y|\hat{z},\hat{x})=P(y|\hat{z}) \textrm{\quad because \quad} (Y\!\perp\!\!\!\perp Z|X)_{G_{\overline{X}\overline{Z}}}.$$
\end{itemize}

Now, by equations (\ref{step1}) and (\ref{step2}),

 $$P(y|\hat{x})= \sum_{z}^{}P(y|z,\hat{x})P(z|\hat{x})=\sum_{z}^{}P(z|x)\sum_{x'}^{}P(y|x',z)P(x').$$

\end{itemize}
\end{proof}
\section*{Acknowledgements}
This work has been partially supported by Office of Naval Research grant ONR N00014-17-1-2842.  This research is based upon work supported in part by the Office of the Director of National
Intelligence (ODNI), Intelligence Advanced Research Projects Activity (IARPA), award/contract
number 2017-16112300009. The views and conclusions contained therein are those of the authors
and should not be interpreted as necessarily representing the official policies, either expressed or implied,
of ODNI, IARPA, or the U.S. Government. The U.S. Government is authorized to reproduce
and distribute reprints for governmental purposes, notwithstanding annotation therein.

\bibliography{references}
\end{document}